\documentclass[letterpaper, 10 pt, conference]{ieeeconf}

\IEEEoverridecommandlockouts %
\makeatletter
\let\NAT@parse\undefined
\makeatother
\usepackage[numbers]{natbib}

\usepackage{amsmath,amssymb}
\usepackage{bm}
\usepackage{xcolor}
\usepackage{graphicx}
\usepackage{mathtools}
\usepackage{hyperref}
\usepackage{cleveref}
\usepackage{siunitx}
\usepackage{tikz}
\usetikzlibrary{shapes}
\usepackage{multirow}
\usepackage{makecell}
\usepackage{cuted}
\usepackage{textcomp}
\usepackage{stfloats}
\usepackage{url}
\usepackage{verbatim}
\usepackage{gensymb}
\usepackage{subcaption}
\usepackage{tablefootnote}
\usepackage{booktabs}
\usepackage{microtype}
\usepackage{balance}
\let\labelindent\relax
\usepackage{enumitem}

\usepackage{amsthm}

\Crefformat{figure}{#2Fig.~#1#3}
\Crefmultiformat{figure}{Figs.~#2#1#3}{ and~#2#1#3}{, #2#1#3}{ and~#2#1#3}
\Crefformat{equation}{#2Eq.~#1#3}

\definecolor{darkred}{HTML}{9B1B30}
\definecolor{grey}{HTML}{959A9C}
\hypersetup{
    colorlinks=true,
    linkcolor=black,
    citecolor=darkred,
    filecolor=blue,
    urlcolor=grey,
}

\newtheorem{theorem}{Theorem}[section]
\newtheorem{lemma}[theorem]{Lemma}

\newtheorem{remark}[theorem]{Remark}

\DeclareMathOperator{\tr}{tr}

\DeclareMathOperator{\Lip}{Lip}
\DeclareMathOperator{\diam}{diam}
\DeclareMathOperator{\dist}{dist}
\DeclareMathOperator{\reach}{reach}
\DeclareMathOperator{\TV}{TV}
\DeclareMathOperator{\Vol}{Vol}
\DeclareMathOperator{\Id}{Id}
\newcommand{\R}{\mathbb{R}}
\newcommand{\E}{\mathbb{E}}
\newcommand{\Prob}{\mathbb{P}}
\newcommand{\calL}{\mathcal{L}}

\title{\LARGE \bf Distributional Stability of Tangent-Linearized \\Gaussian Inference on Smooth Manifolds}
\author{Junghoon Seo, Hakjin Lee, Jaehoon Sim%
\thanks{The authors are with AI Robot Team, PIT IN Corp., South Korea.
        {\tt\footnotesize \{sjh,hj,simjeh\}@pitin-ev.com}}%
}

\begin{document}
\maketitle

\begin{abstract}
Gaussian inference on smooth manifolds is central to robotics, but exact marginalization and
conditioning are generally non-Gaussian and geometry-dependent.
We study tangent-linearized Gaussian inference and
derive explicit non-asymptotic $W_2$ stability bounds for projection marginalization and
surface-measure conditioning.
The bounds separate local second-order geometric distortion from nonlocal tail leakage and, for
Gaussian inputs, yield closed-form diagnostics from $(\mu,\Sigma)$ and curvature/reach surrogates.
Circle and planar-pushing experiments validate the predicted calibration transition near $\sqrt{\|\Sigma\|_{\mathrm{op}}}/R\approx 1/6$ and indicate that normal-direction uncertainty is the dominant failure mode when locality breaks.
These diagnostics provide practical triggers for switching from single-chart linearization to
multi-chart or sample-based manifold inference.
Code and Jupyter notebooks are available at \url{https://github.com/mikigom/StabilityTLGaussian}.
\end{abstract}

\section{Introduction}\label{sec:intro}
Probabilistic inference on smooth manifolds is fundamental in robotics whenever state spaces or
constraints are intrinsically nonlinear (e.g., pose/Lie groups~\cite{barfoot2024state}, contact
sets~\cite{qadri2022incopt,kaess2011isam2}, and kinematic manifolds~\cite{csucan2012motion}).
In these settings, estimators must return both manifold-valued states and calibrated uncertainty
for planning, control, and sensor fusion~\cite{barfoot2024state}.
Two recurring operations are \emph{marginalization} (eliminating nuisance directions) and
\emph{conditioning} (enforcing constraints).
Although both are classical in Euclidean spaces, their manifold counterparts depend on local
geometry and probability-mass locality and are not captured by linear-subspace theory alone~\cite{eustice2006exactly}.
For positive-codimension manifolds, $\{X\in\mathcal M\}$ has zero ambient measure, so conditioning
must instead be defined by restricting the ambient density to $\mathcal M$ and renormalizing
with respect to surface measure, rather than by event conditioning in $\mathbb R^n$.

Recently, Guo et al. (2025)~\cite{guo2025marginalizing} derived explicit Gaussian identities for \emph{linear} constraint manifolds and advocated a linearize--infer--retract workflow for smooth manifolds: linearize at
$\tilde\mu\in\mathcal M$, apply affine-manifold identities on $T_{\tilde\mu}\mathcal M$, and map
the result back to $\mathcal M$ using a retraction or chart.
In robotics estimation, their framework enables tight, geometry-consistent uncertainty extraction for constrained inference by performing closed-form marginalization/conditioning on tangent-space surrogates and retracting back to the manifold.
This viewpoint connects to constrained covariance extraction in factor-graph systems~\cite{kaess2009covariance},
projection-based directional uncertainty models on circles and spheres~\cite{wang2013directional,HernandezStumpfhauser2017TheGP},
and geometry-aware filtering on Lie groups/manifolds~\cite{bourmaud2015continuous,bonnable2009invariant,van2022equivariant}.
Yet the question remains: \textbf{when is a single tangent linearization reliable at the distribution level?}
This question arises concretely when propagating quaternion covariance on $S^3$, extracting contact-consistent uncertainty, or enforcing SLAM loop closures.
Despite the prevalence of tangent-linearized workflows in robotics estimation, no existing work provides explicit, non-asymptotic distributional error bounds that separate the roles of curvature, covariance scale, and mean offset for these operations.

In this paper, we formalize reliability as a \emph{distributional} stability question: how close are the exact manifold-induced laws to the laws produced by tangent linearization and retraction?
We quantify this discrepancy using the $2$-Wasserstein distance $W_2$, since it metrizes weak convergence with second moments and therefore directly controls mean and covariance errors~\cite{villani2008optimal,peyre2019computational}.
Concretely, for an ambient Gaussian $X\sim \mathcal N(\mu,\Sigma)$ and a $C^2$ embedded submanifold $\mathcal M\subset\mathbb R^n$, we take as \emph{exact} targets the two canonical ways of pushing mass onto $\mathcal M$:
\begin{itemize}[leftmargin=*,nosep]
\item \textbf{Marginalization via projection:}
$P_{\mathrm{marg}} := g_\#\mathcal N(\mu,\Sigma)$ where $g$ is (locally) the Euclidean metric projection onto $\mathcal M$.
\item \textbf{Conditioning via surface measure:}
$P_{\mathrm{cond}}$ with density proportional to the ambient Gaussian density restricted to $\mathcal M$,
i.e., $dP_{\mathrm{cond}} \propto \varphi_{\mu,\Sigma}\, d\mathrm{Vol}_{\mathcal M}$.
\end{itemize}
Given a linearization point $\tilde\mu\in\mathcal M$, its surrogate replaces $\mathcal M$ by its tangent space $T_{\tilde\mu}\mathcal M$, applies the corresponding affine-manifold identities, and then maps the resulting law back to $\mathcal M$ via a retraction or chart~\cite{guo2025marginalizing,absil2008optimization}.
Our contributions are threefold:
\begin{itemize}[leftmargin=*,nosep]
\item We derive explicit non-asymptotic $W_2$ stability bounds for projection marginalization
and surface-measure conditioning, decomposing error into local second-order geometric distortion and nonlocal tail leakage.
\item We specialize these bounds to Gaussian inputs, yielding computable diagnostics from
$(\mu,\Sigma)$ together with local curvature/reach surrogates.
\item We validate the predicted regimes in circle and planar-pushing experiments, including
anisotropy, offset, and directional-noise stress tests.
\end{itemize}

\section{Preliminaries}\label{sec:prelim}
We work in the Euclidean space $(\mathbb R^n,\langle\cdot,\cdot\rangle)$ with norm
$\|x\|:=\sqrt{\langle x,x\rangle}$. For $r>0$ and $x\in\mathbb R^n$, let
$B(x,r):=\{z\in\mathbb R^n:\|z-x\|\le r\}$ and $B(r):=B(0,r)$.
For a matrix $A$, $\|A\|_{\mathrm{op}}$ and $\|A\|_F$ denote the operator and Frobenius norms, respectively.
We write $\mathbf 1_A$ for the indicator of an event $A$.
For a measurable map $h:\mathbb R^n\to\mathbb R^m$ and a probability measure $P$ on $\mathbb R^n$,
the \emph{pushforward} is $h_\#P$, defined by $(h_\#P)(B)=P(h^{-1}(B))$ for Borel $B\subset\mathbb R^m$.
If $X$ has law $\mathcal L(X)=P$, then $\mathcal L(h(X))=h_\#P$.

\paragraph{Wasserstein distance}
Let $\mathcal P_2(\mathbb R^n)$ be the set of Borel probability measures on $\mathbb R^n$ with finite second moment.
For $P,Q\in\mathcal P_2(\mathbb R^n)$, the $2$-Wasserstein distance~\cite{villani2008optimal,peyre2019computational} is
\begin{equation}\label{eq:W2def}
W_2(P,Q)
:=
\inf_{\pi\in\Pi(P,Q)}\left(\int_{\mathbb R^n\times\mathbb R^n}\|x-y\|^2\,d\pi(x,y)\right)^{1/2},
\end{equation}
where $\Pi(P,Q)$ denotes the set of couplings of $(P,Q)$.
For any coupling $(X,Y)$ with $\calL(X)=P$ and $\calL(Y)=Q$,
\begin{equation}\label{eq:W2upper}
W_2(P,Q)\le \big(\mathbb E\|X-Y\|^2\big)^{1/2}.
\end{equation}
If $h:\mathbb R^n\to\mathbb R^m$ is $L$-Lipschitz,
\begin{equation}\label{eq:W2LipPush}
W_2(h_\#P,h_\#Q)\le L\,W_2(P,Q).
\end{equation}
Moreover, $W_2$ controls both location ($\|\mu_P-\mu_Q\|\le W_2(P,Q)$) and, via Cauchy--Schwarz, uncentered second-moment discrepancy ($\|M_2(P)-M_2(Q)\|_F \le \sqrt{2(m_2(P)+m_2(Q))}\,W_2(P,Q)$, where $m_2(\nu):=\E_{Z\sim\nu}\|Z\|^2$)~\cite{panaretos2019statistical}.
We therefore state our stability bounds in $W_2$ as practical calibration criteria for both marginalization and conditioning.

\paragraph{Embedded submanifolds, curvature, and reach}
Let $\mathcal M\subset\mathbb R^n$ be a $C^2$ embedded submanifold of dimension $d$.
For $y\in\mathcal M$, denote by $T_y\mathcal M$ its tangent space and by $N_y\mathcal M=(T_y\mathcal M)^\perp$
its normal space. Let $\Pi_{T_y\mathcal M}$ and $\Pi_{N_y\mathcal M}$ be the corresponding orthogonal projectors.
Write $\mathrm{II}_y:T_y\mathcal M\times T_y\mathcal M\to N_y\mathcal M$ for the second fundamental form.
The \emph{reach} of a closed set $S\subset\mathbb R^n$, $\reach(S)$, is the largest $\rho$ such that every point
within distance $<\rho$ has a unique nearest point in $S$.
If $\reach(\mathcal M)\ge \rho>0$, the metric projection $g$ is well-defined on the tube
$\mathcal T_\rho(\mathcal M):=\{x:\dist(x,\mathcal M)<\rho\}$~\cite{federer1959curvature}.

\paragraph{Retractions and charts}
Fix $\tilde\mu\in\mathcal M$ and write $T:=T_{\tilde\mu}\mathcal M$.
A (local) retraction at $\tilde\mu$ is a $C^2$ map $R_{\tilde\mu}:B_T(r)\to\mathcal M$ satisfying
$R_{\tilde\mu}(0)=\tilde\mu$ and $DR_{\tilde\mu}(0)=\Id_T$.
We use a quadratic accuracy bound~\cite{absil2008optimization}: there exists $\kappa_R\ge 0$ such that
\begin{equation}\label{eq:RetractQuadPrelim}
\|R_{\tilde\mu}(v)-(\tilde\mu+v)\|
\le \frac{\kappa_R}{2}\,\|v\|^2,
\qquad \|v\|\le r.
\end{equation}
More generally, for a chart $\Psi:B_r\subset\mathbb R^d\to\mathcal M$ with $\Psi(0)=\tilde\mu$,
the induced volume satisfies $d\mathrm{Vol}_{\mathcal M}(\Psi(v)) = J(v)\,dv$~\cite{lee2006riemannian}.

\subsection{Marginalization and conditioning onto manifolds}
Let $X\in\R^n$ have law $P$ with density $p$, and let
$\mathcal M=\{x\in\R^n:\ f(x)=0\}$ be a $d$-dimensional smooth manifold defined by a regular map
$f:\R^n\to\R^{n-d}$. Following~\cite{guo2025marginalizing}, there are two canonical ways to push
probability onto $\mathcal M$:
(i)~\emph{marginalization by projection}: choose $g_{\mathcal M}:\R^n\to\mathcal M$ and set
\begin{equation}\label{eq:marg-def}
P_{\mathrm{marg}}:=(g_{\mathcal M})_\# P;
\end{equation}
(ii)~\emph{conditioning}: restrict and renormalize, giving
\begin{equation}\label{eq:cond-def}
dP_{\mathrm{cond}}(x)=Z_{\mathcal M}^{-1}\,p(x)\,d\Vol_{\mathcal M}(x),
\qquad
Z_{\mathcal M}:=\int_{\mathcal M}p(x)\,d\Vol_{\mathcal M}(x),
\end{equation}
assuming $0<Z_{\mathcal M}<\infty$.
For $\mathrm{codim}(\mathcal M)>0$, this is \emph{not} event conditioning on
$\{X\in\mathcal M\}$ (which has measure zero under the ambient law). It is a new probability
measure obtained by restricting the density to $\mathcal M$ and normalizing with respect to
$\Vol_{\mathcal M}$.

For an affine linear manifold $\mathcal M_{\mathrm{lin}} := \{x\in\mathbb R^n:\ S^\top x = c\}$ with full-rank
$S\in\mathbb R^{n\times q}$, $q<n$, let $N$ span $\mathrm{null}(S^\top)$ and define $\Pi:=N(N^\top N)^{-1}N^\top$.
For $X\sim\mathcal N(\mu,\Sigma)$, marginalization and conditioning admit closed forms~\cite{guo2025marginalizing}:
\begin{equation}\label{eq:lin-marg}
\mu_{\mathrm{marg}} = \Pi\mu + x_0,
\qquad
\Sigma_{\mathrm{marg}} = \Pi\Sigma\Pi^\top,
\end{equation}
\begin{equation}\label{eq:lin-cond}
\begin{aligned}
\Sigma_{\mathrm{cond}} &= N\big(N^\top\Sigma^{-1}N\big)^{-1}N^\top,
\\
\mu_{\mathrm{cond}} &= x_0 + \Sigma_{\mathrm{cond}}\Sigma^{-1}(\mu-x_0),
\end{aligned}
\end{equation}
where $x_0:=S(S^\top S)^{-1}c\in\mathcal M_{\mathrm{lin}}$.

\subsection{Linearization principle for smooth manifolds}
The closed-form identities \eqref{eq:lin-marg}--\eqref{eq:lin-cond} hold when the constraint manifold is an affine subspace of~$\R^n$.
To extend them to a smooth nonlinear submanifold $\mathcal M\subset\R^n$, we adopt a standard
\emph{tangent-plane surrogate}: select a linearization point $\tilde\mu\in\mathcal M$ (typically
the Euclidean projection of the ambient mean onto $\mathcal M$ or a current estimate) and
replace $\mathcal M$ locally by its first-order approximation at~$\tilde\mu$~\cite{guo2025marginalizing}.
Specifically, let $\mathcal M=\{x:\ f(x)=0\}$ be a regular level set for a smooth
$f:\R^n\to\R^{n-d}$ with $\mathrm{rank}\,Df(\tilde\mu)=n-d$.
The tangent space at $\tilde\mu$ then admits the affine representation
\begin{equation}\label{eq:tangent-affine}
T_{\tilde\mu}\mathcal M
=\{x\in\R^n:\ S^\top x=c\},
\,
S := Df(\tilde\mu)^\top,
\,
c:=S^\top\tilde\mu,
\end{equation}
where $N:=\mathrm{null}(S^\top)$ spans tangent directions and
$\Pi:=N(N^\top N)^{-1}N^\top$ is the orthogonal projector onto $T_{\tilde\mu}\mathcal M$.
Applying \eqref{eq:lin-marg} or \eqref{eq:lin-cond} in this affine setting yields a
\emph{degenerate Gaussian} on $T_{\tilde\mu}\mathcal M$, constituting a first-order approximation
to the on-manifold distribution.

Because tangent-space distributions do not lie on $\mathcal M$ itself, the resulting Gaussian is
mapped back to $\mathcal M$ via a retraction $R_{\tilde\mu}$ or an explicit chart~$\Psi$,
producing the pushforward $\widehat P:= (R_{\tilde\mu})_\# P_T$ (or $\widehat P:=\Psi_\# Q_T$ in
local coordinates), where $P_T$ is the Gaussian on $T_{\tilde\mu}\mathcal M$.
In what follows, we adopt this tangent--retraction construction as the baseline and derive
distribution-level error bounds that quantify when a single linearization suffices.
\section{\texorpdfstring{$W_2$}{W2} stability for Gaussian marginalization}\label{sec:marg}
\begin{figure}[t]
    \centering
    \includegraphics[width=\linewidth]{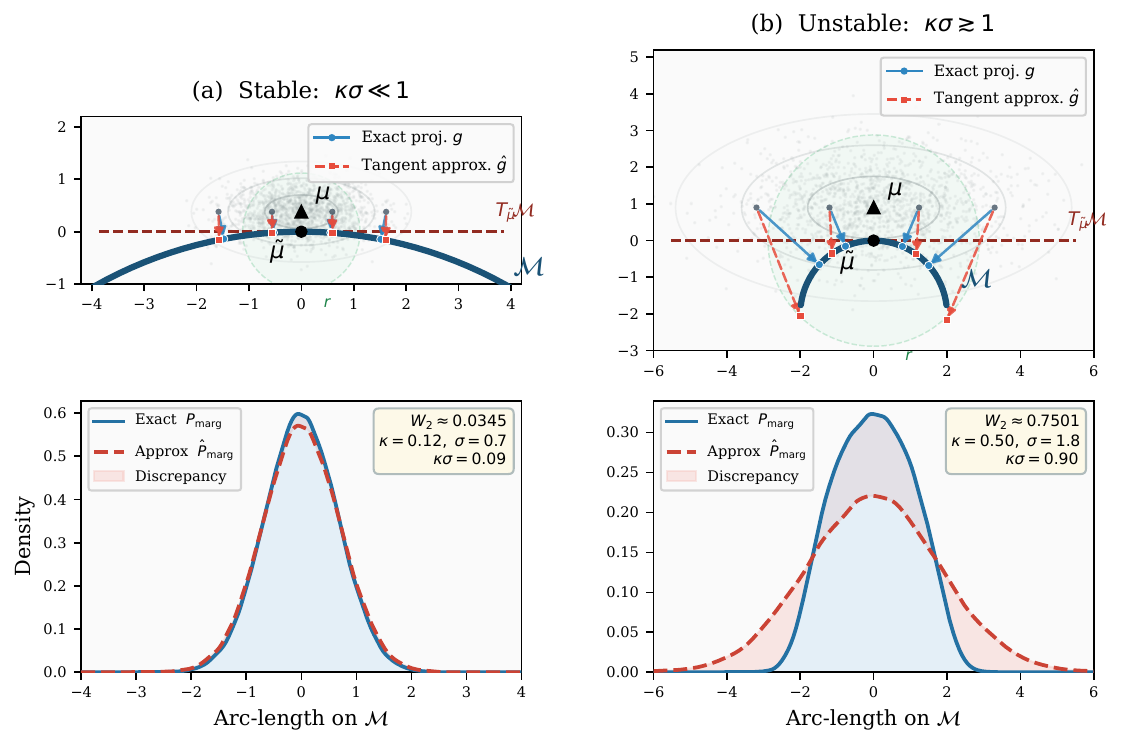}
    \caption{%
    \textbf{$W_2$ stability of marginalizing a Gaussian onto a circular arc.}
    Exact metric projection $g$ vs.\ the tangent--retraction approximation $\hat g$.}
    \label{fig:teaser-marg}
\end{figure}

Let $\mathcal M\subset\mathbb R^n$ be a $C^2$ embedded submanifold and let
$X\sim\mathcal N(\mu,\Sigma)$ with $\Sigma\succ0$.
With $g:\mathbb R^n\to\mathcal M$ denoting a measurable nearest-point projection map (locally, the unique metric projection),
let $Y:=g(X)$ and $P_{\mathrm{marg}}:=g_\#\mathcal N(\mu,\Sigma)$ as in~\eqref{eq:marg-def}.
This section quantifies, in $W_2$, the discrepancy between $P_{\mathrm{marg}}$
and a tangent--retraction surrogate.
Given a linearization point $\tilde\mu\in\mathcal M$, write $T:=T_{\tilde\mu}\mathcal M$ with orthogonal projector $\Pi_T$ and define
\begin{equation}\label{eq:surrogate-marg}
\begin{aligned}
\hat g(x) &:= \bar{R}\big(\Pi_T(x-\tilde\mu)\big),\\
\widehat Y &:= \hat g(X),\\
\widehat P_{\mathrm{marg}} &:= \mathcal L(\widehat Y).
\end{aligned}
\end{equation}
where $\bar R:T\to\mathcal M$ is a measurable extension of a local retraction
$R_{\tilde\mu}:B_T(r)\to\mathcal M$ outside $B_T(r)$, chosen so that the fourth moment in $C_{\mathrm{tail}}$ below is finite (for instance, by using a constant extension outside $B_T(r)$).
The retraction satisfies the quadratic accuracy bound \eqref{eq:RetractQuadPrelim}.
The main estimate below separates a local second-order geometric error from a tail contribution
outside the localization radius $r$.

\paragraph{Geometric locality}
Assume $\reach(\mathcal M)\ge\rho>0$ and fix $r\in(0,\rho/2)$.
Under this assumption, $g$ coincides with the unique metric projection on the tube $\mathcal T_\rho(\mathcal M)$.
Assume a local curvature bound on $\mathcal M\cap B(\tilde\mu,2r)$:
\begin{equation}\label{eq:lem-curv}
\sup_{y\in\mathcal M\cap B(\tilde\mu,2r)}\|\mathrm{II}_y\|_{\mathrm{op}}\le \kappa.
\end{equation}

\begin{theorem}[W$_2$ stability of marginalization under tangent--retraction]\label{thm:W2-marg}
There exists a constant $C_{\mathrm{loc}}>0$ (depending only on dimensionless tube/curvature margins,
e.g.\ $r/\rho$ and $r\kappa$, and on the ambient dimension) such that, with
$A:=\{\|X-\tilde\mu\|\le r\}$ and $\varepsilon:=\Prob(A^c)$,
\begin{equation}\label{eq:W2MargMain}
\begin{aligned}
W_2\!\big(P_{\mathrm{marg}},\widehat P_{\mathrm{marg}}\big)
&\le
C_{\mathrm{loc}}(\kappa+\kappa_R)\Big(\E\big[\|X-\tilde\mu\|^4\mathbf 1_A\big]\Big)^{1/2}
\\
&\quad+\ C_{\mathrm{tail}}\;\varepsilon^{1/4},
\end{aligned}
\end{equation}
where
\[
C_{\mathrm{tail}}
:=
\big(\E\|Y-\tilde\mu\|^4\big)^{1/4}
+
\big(\E\|\widehat Y-\tilde\mu\|^4\big)^{1/4}
<\infty.
\]
\end{theorem}

\begin{remark}[Practical calibration of $C_{\mathrm{loc}}$]
$C_{\mathrm{loc}}$ depends only on local geometry, not on $(\mu,\Sigma)$.
It can be estimated by sampling $x_j\in B(\tilde\mu,r)$, computing the ratio $\|g(x_j)-R_{\tilde\mu}(\Pi_T(x_j-\tilde\mu))\|/\big((\kappa+\kappa_R)\|x_j-\tilde\mu\|^2\big)$, and taking a high quantile with a safety factor.
\end{remark}

\begin{proof}[Proof of Theorem~\ref{thm:W2-marg}]
We couple $Y=g(X)$ and $\widehat Y=\hat g(X)$ through the same Gaussian input $X$, and we estimate
$\E\|Y-\widehat Y\|^2$ by separating a local region (where geometry is second-order accurate) from its complement.
The deterministic part is a local comparison between metric projection and the tangent--retraction map.

\begin{lemma}[Local quadratic comparison of projection and tangent--retraction]\label{lem:local-proj-retr}
Under the geometric locality assumptions (Equation \eqref{eq:lem-curv}),
let $R_{\tilde\mu}:B_T(r)\to\mathcal M$ be a $C^2$ retraction with \eqref{eq:RetractQuadPrelim}.
Then there exists $C_{\mathrm{loc}}>0$ such that for every $x$ with $\|x-\tilde\mu\|\le r$,
\begin{equation}\label{eq:lem-main}
\big\|g(x)-R_{\tilde\mu}\big(\Pi_T(x-\tilde\mu)\big)\big\|
\ \le\
C_{\mathrm{loc}}\,(\kappa+\kappa_R)\,\|x-\tilde\mu\|^2.
\end{equation}
\end{lemma}

\begin{proof}
Fix $x$ with $d:=\|x-\tilde\mu\|\le r$, $y:=g(x)\in\mathcal M$, and $v:=\Pi_T(x-\tilde\mu)\in T$.
Since $r<\rho$, the metric projection satisfies $\Lip(g|_{\mathcal T_r})\le \rho/(\rho-r)\le 2$, giving $\|y-\tilde\mu\|\le 2d$.
Standard $C^2$ graph estimates~\cite{boissonnat2021triangulating,niyogi2008finding} yield, for $z\in\mathcal M\cap B(\tilde\mu,2r)$,
\begin{align}
\|\Pi_{T^\perp}(z-\tilde\mu)\|
&\le \tfrac{C_{\mathrm{ht}}\kappa}{2}\,\|z-\tilde\mu\|^2,
\label{eq:height-est}\\
\|\Pi_{T_z\mathcal M}-\Pi_T\|_{\mathrm{op}}
&\le C_{\mathrm{ang}}\kappa\,\|z-\tilde\mu\|.
\label{eq:angle-est}
\end{align}
Since $\Pi_{T_y\mathcal M}(x-y)=0$, applying~\eqref{eq:angle-est} at $z=y$ gives $\|v-\Pi_T(y-\tilde\mu)\|\le 3C_{\mathrm{ang}}\kappa\, d^2$, and~\eqref{eq:height-est} gives $\|\Pi_{T^\perp}(y-\tilde\mu)\|\le 2C_{\mathrm{ht}}\kappa\, d^2$.
Hence $\|y-(\tilde\mu+v)\|\le C_{\mathrm{proj}}\kappa\, d^2$ with $C_{\mathrm{proj}}:=2C_{\mathrm{ht}}+3C_{\mathrm{ang}}$, and the retraction bound~\eqref{eq:RetractQuadPrelim} completes the triangle inequality to~\eqref{eq:lem-main}. \qedhere
\end{proof}

Returning to the random setting, the shared-input coupling and~\eqref{eq:W2upper} give
\begin{equation}\label{eq:W2-couple}
W_2\big(\calL(Y),\calL(\widehat Y)\big)
\ \le\ \Big(\E\|Y-\widehat Y\|^2\Big)^{1/2}.
\end{equation}
Let $A:=\{\|X-\tilde\mu\|\le r\}$. Splitting the second moment over $A$ and $A^c$,
\begin{equation}\label{eq:split}
\E\|Y-\widehat Y\|^2
:=
\E\big[\|Y-\widehat Y\|^2\mathbf 1_A\big]
+
\E\big[\|Y-\widehat Y\|^2\mathbf 1_{A^c}\big].
\end{equation}
On $A$, Lemma~\ref{lem:local-proj-retr} yields
\[
\|Y-\widehat Y\|
\le C_{\mathrm{loc}}(\kappa+\kappa_R)\|X-\tilde\mu\|^2,
\]
therefore
\begin{equation}\label{eq:local-term}
\begin{aligned}
\Big(\E\big[\|Y-\widehat Y\|^2\mathbf 1_A\big]\Big)^{1/2}
&\le C_{\mathrm{loc}}(\kappa+\kappa_R)\\
&\quad\times \Big(\E\big[\|X-\tilde\mu\|^4\mathbf 1_A\big]\Big)^{1/2}.
\end{aligned}
\end{equation}
On $A^c$, setting $Z:=\|Y-\tilde\mu\|+\|\widehat Y-\tilde\mu\|$ and applying H\"older ($L^4/L^{4/3}$) followed by Minkowski gives $(\E[\|Y-\widehat Y\|^2\mathbf 1_{A^c}])^{1/2}\le C_{\mathrm{tail}}\,\varepsilon^{1/4}$.
Combining with \eqref{eq:W2-couple}, \eqref{eq:split}, and \eqref{eq:local-term}
yields \eqref{eq:W2MargMain}. \qedhere
\end{proof}

\paragraph{Gaussian specialization}
Set $\delta:=\mu-\tilde\mu$ and write $X-\tilde\mu=\delta+\xi$ with $\xi\sim\mathcal N(0,\Sigma)$.
Then the \emph{untruncated} fourth moment in \eqref{eq:W2MargMain} has the closed form
\begin{equation}
\mathbb E\|X-\tilde\mu\|^4
=\|\delta\|^4
+2\|\delta\|^2\,\mathrm{tr}\,\Sigma
+4\,\delta^\top\Sigma\,\delta
+(\mathrm{tr}\,\Sigma)^2
+2\|\Sigma\|_F^2.
\label{eq:GaussianFourthMoment}
\end{equation}
On $A:=\{\|X-\tilde\mu\|\le r\}$,
the localized fourth moment satisfies
$\mathbb E[\|X-\tilde\mu\|^4\,\mathbf 1_A]\le r^4$ and can be sharpened via the CDF of $\|X-\tilde\mu\|^2$.

For the tail probability $\varepsilon:=\mathbb P(A^c)$, let $\lambda_{\max}:=\|\Sigma\|_{\mathrm{op}}$
and $Z\sim\mathcal N(0,I_n)$. Using $\mathbb P(\|Z\|\ge \sqrt n+t)\le e^{-t^2/2}$~\cite{vershynin2009high} yields, for any
$r>\|\delta\|$,
\begin{equation}
\varepsilon
=\mathbb P(\|X-\tilde\mu\|>r)
\le \exp\!\left(
-\frac12\Big(\frac{r-\|\delta\|}{\sqrt{\lambda_{\max}}}-\sqrt n\Big)_+^2
\right).
\label{eq:GaussianTail}
\end{equation}
Choosing $r=\|\delta\|+\sqrt{\lambda_{\max}}(\sqrt n+t)$ gives $\varepsilon\le e^{-t^2/2}$.
The constant $C_{\mathrm{tail}}$ can be made fully explicit for a nearest-point $g$ and a bounded or at-most-quadratic extension $\bar R$: Lipschitz projection inside the tube ($\mathrm{Lip}(g)\le\rho/(\rho-r)$), nearest-point control outside it, and retraction accuracy~\eqref{eq:RetractQuadPrelim} bound the relevant fourth-moment terms by Gaussian moments of $\|X-\tilde\mu\|$, yielding a computable criterion from $(\mu,\Sigma)$ and $(\rho,\kappa,\kappa_R,r)$.

\paragraph{Interpretation and numerical illustration}
Equation~\eqref{eq:W2MargMain} decomposes the error into a local geometric term and a tail-leakage term.
Inside $A=\{\|X-\tilde\mu\|\le r\}$, the mismatch between $g$ and $\hat g$ is second order,
scaling as $(\kappa+\kappa_R)\|X-\tilde\mu\|^2$ (Lemma~\ref{lem:local-proj-retr}).
For Gaussian inputs, the effective spread scale is $\|\delta\|+\sqrt{\lambda_{\max}}\sqrt n$, so the local term is
roughly $(\kappa+\kappa_R)(\|\delta\|+\sqrt{\lambda_{\max}}\sqrt n)^2$.
The tail term is controlled by $\varepsilon=\Prob(\|X-\tilde\mu\|>r)$ and decays exponentially once
$r\gtrsim \|\delta\|+\sqrt{\lambda_{\max}}(\sqrt n+t)$, but can dominate when reach forces small $r$.
Hence, $r$ must balance probability-mass coverage against local geometric validity.
Figure~\ref{fig:teaser-marg} supports this prediction: in local regimes, the exact and
approximate laws nearly coincide. Beyond the transition, linearization becomes miscalibrated and
typically overconfident.
The spectral scale $\|\Sigma\|_{\mathrm{op}}$ is a worst-case proxy for how much probability mass can leave the local tube.
In the projection geometries tested below, uncertainty with a large normal component is the dominant failure mode because it stresses curvature and reach most directly.
The height and angle estimates~\eqref{eq:height-est}--\eqref{eq:angle-est} still control general local displacements, so tangent-dominant uncertainty is covered by the theorem but can make the spectral bound conservative (confirmed in Section~\ref{sec:numerics}).

\section{\texorpdfstring{$W_2$}{W2} stability for Gaussian conditioning}\label{sec:cond}
\begin{figure}[t]
    \centering
    \includegraphics[width=\linewidth]{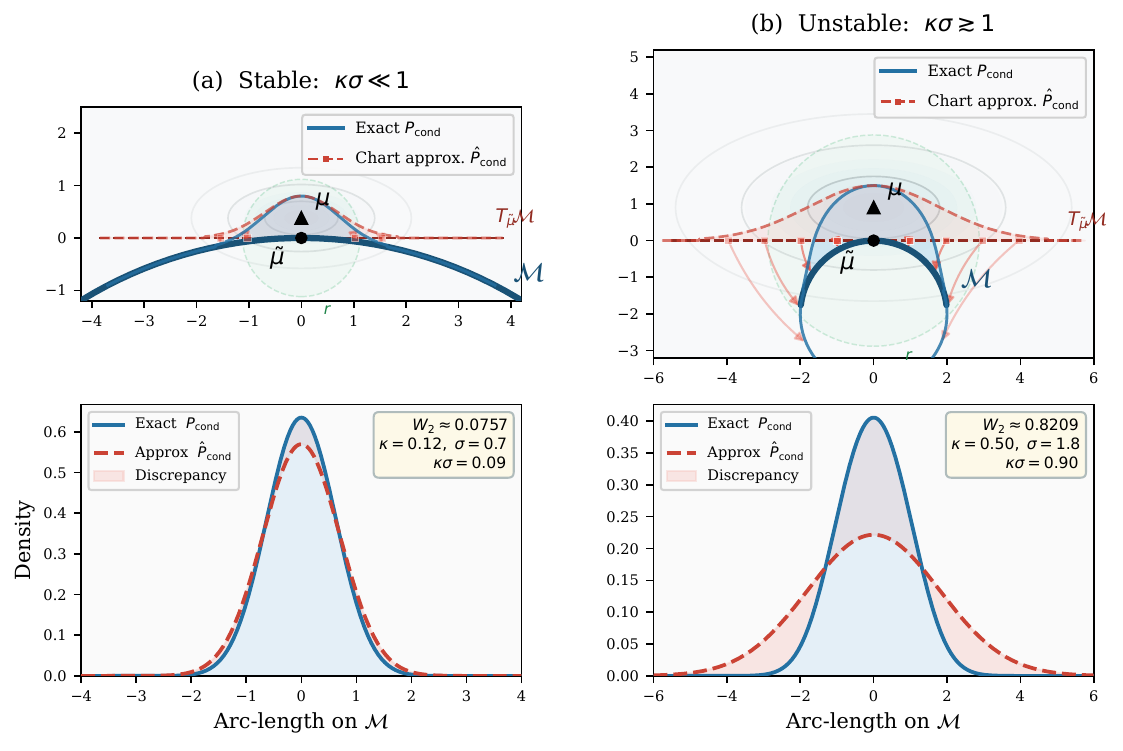}
    \caption{%
    \textbf{$W_2$ stability of conditioning a Gaussian onto a circular arc.}
    Surface-measure conditioning vs.\ a tangent-plane chart approximation.}
    \label{fig:teaser-cond}
\end{figure}

Unlike marginalization, conditioning depends on density shape on $\mathcal{M}$, not just mass transport. We therefore replace the direct coupling of Section~\ref{sec:marg} with likelihood-ratio and total-variation arguments.
Let $\mathcal M\subset\mathbb R^n$ be a $C^2$ embedded submanifold of dimension $d$ and
$X\sim\mathcal N(\mu,\Sigma)$ with density $\varphi_{\mu,\Sigma}$.
Assume the surface-measure restriction has $0<Z_{\mathcal M}<\infty$ and finite second moment.
This section quantifies, in $W_2$, the discrepancy between the exact conditioned law
$P_{\mathrm{cond}}$ from~\eqref{eq:cond-def} (with $\varphi=\varphi_{\mu,\Sigma}$)
and an approximation that first conditions on $T_{\tilde\mu}\mathcal M$
and then maps back to $\mathcal M$ through a local chart.

\paragraph{Tangent-plane reference law}
Fix $\tilde\mu\in\mathcal M$ and let $T:=T_{\tilde\mu}\mathcal M$ as in Section~\ref{sec:marg}.
Let $N\in\mathbb R^{n\times d}$ have orthonormal columns spanning $T$ and define the affine parametrization
$x_{\mathrm{lin}}(v)=\tilde\mu+Nv$ for $v\in\mathbb R^d$.
Let $\Omega:=\Sigma^{-1}$ denote the precision matrix, and recall $\delta=\mu-\tilde\mu$ from Section~\ref{sec:marg}.
Define the probability measure $Q_T$ on $\mathbb R^d$ by
\begin{equation}\label{eq:QTDef}
\begin{aligned}
dQ_T(v)
&:=
\frac{\varphi_{\mu,\Sigma}(\tilde\mu+Nv)}{Z_T}\,dv,
\\
Z_T
&:=
\int_{\mathbb R^d}\varphi_{\mu,\Sigma}(\tilde\mu+Nv)\,dv\in(0,\infty).
\end{aligned}
\end{equation}
Equivalently, $Q_T=\mathcal N(m_T,\Sigma_T)$ with
\begin{equation}\label{eq:QTParams}
\Sigma_T=(N^\top \Omega N)^{-1},
\qquad
m_T=\Sigma_T N^\top \Omega\delta.
\end{equation}

\paragraph{Chart model and approximation}
Fix $r>0$ and write $B_r:=\{v\in\mathbb R^d:\ |v|\le r\}$.
Assume $\Psi:B_r\to\mathcal M$ is a $C^2$ embedding onto $\mathcal M_r:=\Psi(B_r)$ with
$\Psi(0)=\tilde\mu$, $D\Psi(0)=N$, and $\Lip(\Psi)\le L_\Psi$ on $B_r$.
Assume the second-order residual bound
\begin{equation}\label{eq:ChartResidual}
\big\|\Psi(v)-(\tilde\mu+Nv)\big\|\le \frac{\kappa_\Psi}{2}\,|v|^2,\qquad v\in B_r,
\end{equation}
\noindent and a quadratic control on the Jacobian factor $J(v)$ defined by
$d\mathrm{Vol}_{\mathcal M}(\Psi(v))=J(v)\,dv$:
\begin{equation}\label{eq:JacobianResidual}
|\log J(v)|\le \frac{\kappa_J}{2}\,|v|^2,\qquad v\in B_r,
\end{equation}
Let $\bar\Psi:\mathbb R^d\to\mathcal M$ be a measurable extension of $\Psi$ such that
$\bar\Psi(v)\in\mathcal M_r$ iff $v\in B_r$ and $\tau_{\widehat P}(r)<\infty$ (e.g., when $\mathcal M\setminus\mathcal M_r\neq\emptyset$, map $v\notin B_r$ to any fixed point in $\mathcal M\setminus\mathcal M_r$),
and define the approximate conditioned law
\begin{equation}\label{eq:PcondHatDef}
\widehat P_{\mathrm{cond}}:=\bar\Psi_\# Q_T.
\end{equation}

\paragraph{Tail quantities}
Define
\begin{equation}\label{eq:tail-probs}
\varepsilon_P:=P_{\mathrm{cond}}(\mathcal M\setminus\mathcal M_r),
\qquad
\varepsilon_Q:=Q_T(\mathbb R^d\setminus B_r),
\end{equation}
and the corresponding tail second moments (w.r.t.\ $\tilde\mu$)
\begin{equation}\label{eq:tail-moments}
\begin{aligned}
\tau_P(r)
&:=\Big(\E\big[\|Y-\tilde\mu\|^2\,\mathbf 1_{\{Y\notin\mathcal M_r\}}\big]\Big)^{1/2},
\quad (Y\sim P_{\mathrm{cond}}),
\\
\tau_{\widehat P}(r)
&:=\Big(\E\big[\|\widehat Y-\tilde\mu\|^2\,\mathbf 1_{\{\widehat Y\notin\mathcal M_r\}}\big]\Big)^{1/2},
\quad (\widehat Y\sim \widehat P_{\mathrm{cond}}).
\end{aligned}
\end{equation}

Define the curvature/volume mismatch exponent
\begin{equation}\label{eq:etarDef}
\eta_r
:=
\frac{\|\Omega\|\kappa_\Psi}{2}\,(\|\delta\|+r)\,r^2
+
\frac{\|\Omega\|\kappa_\Psi^2}{8}\,r^4
+
\frac{\kappa_J}{2}\,r^2.
\end{equation}

\begin{theorem}[W$_2$ stability of surface-measure conditioning under chart linearization]\label{thm:W2-cond}
Under the assumptions above,
\begin{equation}\label{eq:W2CondMain}
\begin{aligned}
W_2\big(P_{\mathrm{cond}},\widehat P_{\mathrm{cond}}\big)
&\le
\tau_P(r)+\tau_{\widehat P}(r)
\ +\ L_\Psi r\big(\sqrt{\varepsilon_P}+\sqrt{\varepsilon_Q}\big)
\\
&\quad+\ 2L_\Psi r\sqrt{\tanh(\eta_r)}.
\end{aligned}
\end{equation}
\end{theorem}

\begin{lemma}[Truncation bound]\label{lem:truncation}
Let $P$ be a probability measure on a subset of $\mathbb R^m$ with finite second moment.
Let $A$ be measurable with $P(A)=1-\varepsilon\in(0,1]$, and let $P(\cdot\mid A)$ be the conditional law.
Fix any $x_0\in A$ and define $R_A:=\sup_{x\in A}\|x-x_0\|$ (possibly $+\infty$).
Then
\begin{equation}\label{eq:trunc-bound}
\begin{aligned}
W_2\big(P,\;P(\cdot\mid A)\big)
&\le
\Big(\E\big[\|X-x_0\|^2\mathbf 1_{A^c}\big]\Big)^{1/2}
\\
&\quad+\ R_A\sqrt{\varepsilon},
\qquad (X\sim P).
\end{aligned}
\end{equation}
\end{lemma}

\begin{proof}
Define the auxiliary measure $\tilde P:=(1-\varepsilon)\,P(\cdot\mid A)+\varepsilon\,\delta_{x_0}$.
Couple $P$ with $\tilde P$ by keeping the mass on $A$ fixed and sending all $A^c$ mass to $x_0$.
By~\eqref{eq:W2upper}, this gives
\[
W_2(P,\tilde P)
\le
\Big(\E\big[\|X-x_0\|^2\mathbf 1_{A^c}\big]\Big)^{1/2}.
\]
Next, couple $\tilde P$ to $P(\cdot\mid A)$ by moving the extra point mass $\varepsilon\,\delta_{x_0}$ into $A$
according to $P(\cdot\mid A)$, with cost at most $R_A\sqrt{\varepsilon}$.
The claim follows by the triangle inequality. \qedhere
\end{proof}

\begin{lemma}[Diameter--TV control of $W_2$ on bounded support]\label{lem:diam-tv}
Let $p,q$ be probability measures supported on a set $S\subset\mathbb R^d$ with
$\diam(S):=\sup_{u,v\in S}\|u-v\|<\infty$.
Then
\begin{equation}\label{eq:diam-tv-bound}
\begin{aligned}
W_2(p,q) &\le \diam(S)\,\sqrt{\TV(p,q)},
\\
\TV(p,q) &:= \frac12\int|dp-dq|.
\end{aligned}
\end{equation}
\end{lemma}

\begin{proof}
Let $\gamma$ be a maximal coupling of $(p,q)$ so that $\gamma\{(U,V):U\neq V\}=\TV(p,q)$~\cite{thorisson2000coupling}.
On $\{U=V\}$ the cost is $0$, and on $\{U\neq V\}$ we have $\|U-V\|\le \diam(S)$.
Thus $\E_\gamma\|U-V\|^2\le \diam(S)^2\,\TV(p,q)$, and the claim follows by infimizing over couplings. \qedhere
\end{proof}

\begin{lemma}[TV control from a likelihood-ratio bound]\label{lem:tv-lr}
Let $p,q$ be probability measures with $p\ll q$ and
\[
\frac{dp}{dq}\in[\Lambda^{-1},\Lambda]\qquad q\text{-a.e.},
\]
for some $\Lambda\ge 1$. Then
\[
\TV(p,q)\le \frac{\Lambda-1}{\Lambda+1}.
\]
\end{lemma}

\begin{proof}
Write $h:=dp/dq\in[\Lambda^{-1},\Lambda]$, so $\TV(p,q)=\frac12\E_q[|h-1|]$ with $\E_q[h]=1$.
Since $x\mapsto|x-1|$ is convex on $[\Lambda^{-1},\Lambda]$, linear interpolation gives
$\E_q[|h-1|]\le 2(\Lambda-1)/(\Lambda+1)$, hence the result.
\qedhere
\end{proof}

\begin{proof}[Proof of Theorem~\ref{thm:W2-cond}]
Recall $P_{\mathrm{cond}}$ from \eqref{eq:cond-def},
$Q_T$ from \eqref{eq:QTDef}, and
$\widehat P_{\mathrm{cond}}=\bar\Psi_\# Q_T$ from \eqref{eq:PcondHatDef}.
Let $\mathcal M_r=\Psi(B_r)$, and define truncations
\[
P_{\mathrm{cond}}^{(r)}:=P_{\mathrm{cond}}(\cdot\mid \mathcal M_r),
\qquad
\widehat P_{\mathrm{cond}}^{(r)}:=\widehat P_{\mathrm{cond}}(\cdot\mid \mathcal M_r).
\]
We first separate tail and interior contributions. By the triangle inequality,
\begin{equation}\label{eq:tri-cond}
\begin{aligned}
W_2(P_{\mathrm{cond}},\widehat P_{\mathrm{cond}})
&\le
W_2(P_{\mathrm{cond}},P_{\mathrm{cond}}^{(r)})
+
W_2(P_{\mathrm{cond}}^{(r)},\widehat P_{\mathrm{cond}}^{(r)})
\\
&\quad+
W_2(\widehat P_{\mathrm{cond}}^{(r)},\widehat P_{\mathrm{cond}}).
\end{aligned}
\end{equation}
Lemma~\ref{lem:truncation} with $A=\mathcal M_r$, $x_0=\tilde\mu$, and $\sup_{x\in\mathcal M_r}\|x-\tilde\mu\|\le L_\Psi r$ bounds the two outer (tail) terms:
\begin{equation}\label{eq:tail-P-simplified}
W_2(P_{\mathrm{cond}},P_{\mathrm{cond}}^{(r)})
\le
\tau_P(r)+L_\Psi r\sqrt{\varepsilon_P},
\end{equation}
\begin{equation}\label{eq:tail-Phat}
W_2(\widehat P_{\mathrm{cond}}^{(r)},\widehat P_{\mathrm{cond}})
\le
\tau_{\widehat P}(r)+L_\Psi r\sqrt{\varepsilon_Q}.
\end{equation}

For the interior, pull both truncated laws back to $B_r$ via the bijection $\Psi$.
Define unnormalized densities $\tilde p(v):=\varphi_{\mu,\Sigma}(\Psi(v))\,J(v)$ and
$\tilde q(v):=\varphi_{\mu,\Sigma}(\tilde\mu+Nv)$, and let $p_r:=\tilde p/\int_{B_r}\tilde p$, $q_r:=\tilde q/\int_{B_r}\tilde q$.
Then $P_{\mathrm{cond}}^{(r)}=\Psi_\# p_r$ and $\widehat P_{\mathrm{cond}}^{(r)}=\Psi_\# q_r$.
By~\eqref{eq:W2LipPush} and Lemma~\ref{lem:diam-tv} ($\diam(B_r)\le 2r$),
\[
W_2(P_{\mathrm{cond}}^{(r)},\widehat P_{\mathrm{cond}}^{(r)})
\le
2L_\Psi r\,\sqrt{\TV(p_r,q_r)}.
\]
Write $e(v):=\Psi(v)-(\tilde\mu+Nv)$.
By~\eqref{eq:ChartResidual}, $\|e(v)\|\le \frac{\kappa_\Psi}{2}r^2$ on $B_r$.
Expanding the log-density difference and applying~\eqref{eq:JacobianResidual} gives $|\log \tilde p(v)-\log \tilde q(v)|\le \eta_r$ on $B_r$.
After normalization, $p_r/q_r\in[e^{-2\eta_r},e^{2\eta_r}]$, and Lemma~\ref{lem:tv-lr} yields $\TV(p_r,q_r)\le \tanh(\eta_r)$.
Inserting the tail and interior bounds into~\eqref{eq:tri-cond} gives~\eqref{eq:W2CondMain}. \qedhere
\end{proof}

\paragraph{Gaussian specialization}
For $Q_T=\mathcal N(m_T,\Sigma_T)$ from~\eqref{eq:QTParams}, the tail $\varepsilon_Q:=Q_T(\mathbb R^d\setminus B_r)$ is bounded by Gaussian concentration as in~\eqref{eq:GaussianTail} with $(\delta,\Sigma)\to(m_T,\Sigma_T)$.
If the extension satisfies $\|\bar\Psi(v)-\tilde\mu\|\le L_{\mathrm{ext}}(1+\|v\|)$ on $\mathbb R^d\setminus B_r$, then Cauchy--Schwarz gives $\tau_{\widehat P}(r)\le L_{\mathrm{ext}}\,\big(\mathbb E(1+\|V\|)^4\big)^{1/4}\varepsilon_Q^{1/4}$; for a constant tail extension to $z_0$, one may instead use $\tau_{\widehat P}(r)\le\|z_0-\tilde\mu\|\sqrt{\varepsilon_Q}$.
The Gaussian fourth moment appearing here has the same form as~\eqref{eq:GaussianFourthMoment}.
The remaining $(\varepsilon_P,\tau_P(r))$ quantify the true conditioned mass outside $\mathcal M_r$.

\paragraph{Interpretation and numerical illustration}
Equation~\eqref{eq:W2CondMain} splits conditioning error into tail/truncation penalties ($\tau_P,\tau_{\widehat P},\varepsilon_P,\varepsilon_Q$ terms) and interior chart distortion $2L_\Psi r\sqrt{\tanh(\eta_r)}$, where $\eta_r$ combines second-order chart residual and Jacobian distortion weighted by $\|\Omega\|$.
For affine manifolds ($\kappa_\Psi=\kappa_J=0$), $\eta_r=0$ and only truncation remains.
In robotics, enforcing hard constraints (e.g.\ unit-quaternion normalization, contact membership) is conditioning, and $\eta_r$ quantifies how constraint curvature and prior precision jointly warp the tangent approximation.
When $\eta_r\ll 1$, a single tangent chart is faithful, but as $\eta_r$ grows the chart distortion dominates and manifests as miscalibrated credible regions on $\mathcal{M}$.
Figure~\ref{fig:teaser-cond} illustrates this transition: near-local regimes show close agreement,
whereas larger spread introduces mode/variance mismatch and then tail-dominated failure.

\section{Numerical experiments}\label{sec:numerics}
\subsection{2-D circle benchmark}
\paragraph{Goals and setup}
We consider $\mathcal M=\{x\in\mathbb R^2:\|x\|=R\}$ with curvature $\kappa=1/R$, reach $\rho=R$, and
$X\sim\mathcal N(\mu,\Sigma)$ where $\mu=(R+\delta,0)$.
The baseline uses isotropic covariance $\Sigma=\sigma^2 I_2$ with $\delta=0.2$.
We then add two stress tests to assess robustness: (i) anisotropic covariance,
$\Sigma=\mathrm{diag}(\sigma_n^2,\sigma_t^2)$ in the normal/tangential frame, and
(ii) offset sweeps in $\delta/R$.
These tests probe the computable diagnostics in Theorem~\ref{thm:W2-marg}, especially the roles of
$\sqrt{\|\Sigma\|_{\mathrm{op}}}$ and the tail-leakage term $\varepsilon$.
For marginalization, the exact map is $g(x)=R\,x/\|x\|$, and the approximation is tangent projection plus
normalization at $\tilde\mu=g(\mu)=(R,0)$.
For conditioning, we compare surface-measure conditioning on
$\mathcal M$ against tangent-line conditioning followed by retraction.

\paragraph{Quantitative sweep}
Figure~\ref{fig:overconf-metrics} sweeps $\sigma/R\in[0.02,1.2]$ for $R\in\{0.5,1,2\}$
and reports (i) variance ratio $\varrho=\mathrm{Var}_{\mathrm{lin}}(\theta)/\mathrm{Var}_{\mathrm{exact}}(\theta)$,
(ii) coverage of a nominal $95\%$ interval, and (iii) normalized theoretical/empirical
$W_2$ diagnostics for marginalization.
Panel~(a) shows $\varrho$ beginning to drop after the $\sigma/R=1/6$ locality marker and crossing below $1$ only at larger spread. Panel~(b) shows corresponding
coverage degradation outside the local regime. Panel~(c) uses
$\sqrt{\|\Sigma\|_{\mathrm{op}}}/R$ and compares empirical $W_2$ proxies against tightened
theoretical $W_2$ bounds with data-calibrated $C_{\mathrm{tail}}$
(90\% quantile pre-fit with a strict upper-envelope safeguard),
so the bound remains above the empirical proxy while preserving the failure-onset trend.

\begin{figure}[t]
  \centering
  \includegraphics[width=\linewidth]{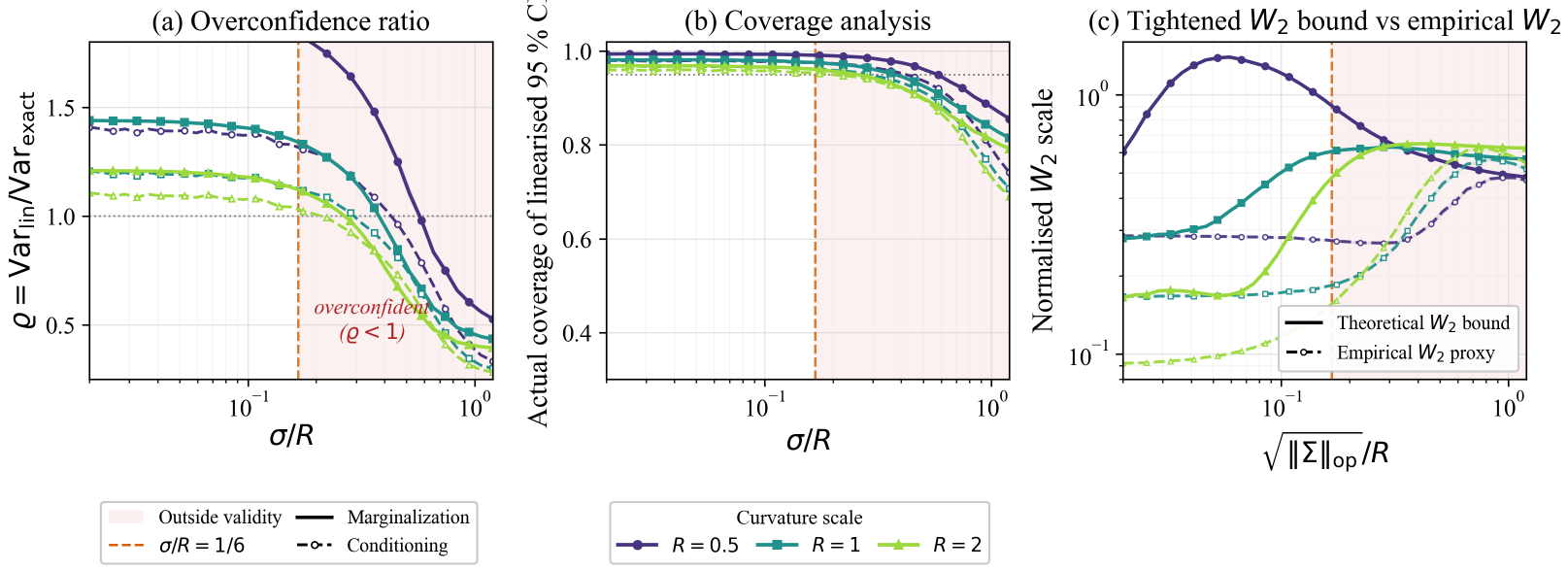}
  \caption{Circle sweep ($\delta=0.2$) over $\sigma/R$:
  (a) variance ratio $\varrho=\mathrm{Var}_{\mathrm{lin}}/\mathrm{Var}_{\mathrm{exact}}$,
  (b) realized $95\%$ coverage, and
  (c) normalized diagnostics on $\sqrt{\|\Sigma\|_{\mathrm{op}}}/R$:
  tightened theoretical $W_2$ bound
  (data-calibrated $C_{\mathrm{tail}}$ with strict upper-envelope safeguard) and empirical $W_2$ proxy.
  Dashed line marks $\sigma/R=1/6$. Colors indicate $R\in\{0.5,1,2\}$.}
  \label{fig:overconf-metrics}
\end{figure}

\paragraph{Anisotropic $\Sigma$ and offset sweeps}
Figure~\ref{fig:overconf-generality} extends the circle benchmark along two complementary axes.
Panel~(a) varies anisotropy $\eta:=\sigma_n/\sigma_t\in\{0.5,1,2,4\}$ and replots calibration against
the spectral scale $\sqrt{\|\Sigma\|_{\mathrm{op}}}/R$.
The transition remains aligned across anisotropy levels, supporting $\|\Sigma\|_{\mathrm{op}}$ as an effective
diagnostic variable in the bound.
Panels~(b)--(c) fix anisotropy and sweep $\delta/R$:
as offset increases, the tail proxy $\varepsilon$ rises rapidly and coverage degrades even at fixed
curvature/noise scale, revealing a second failure axis beyond curvature-only effects.
This behavior matches the decomposition in Theorem~\ref{thm:W2-marg}, where local distortion and
nonlocal leakage contribute separately.
Formally, $\|\delta\|$ enters both \eqref{eq:GaussianFourthMoment} and \eqref{eq:GaussianTail}, confirming mean offset as a failure axis independent of $\sqrt{\|\Sigma\|_{\mathrm{op}}}/\rho$.

\paragraph{Conditioning comparison}
Applying the same circle benchmark to conditioning (Figure~\ref{fig:teaser-cond}), the tangent-chart approximation breaks down \emph{before} its marginalization counterpart: the conditioning variance ratio approaches unity and crosses into the below-unity, overconfident regime at a smaller $\sigma/R$ than marginalization (${\approx}0.30$ vs.\ ${\approx}0.38$), consistent with Theorem~\ref{thm:W2-cond} where $\kappa_J$ and $\|\Omega\|\kappa_\Psi$ jointly amplify chart distortion beyond what curvature alone causes for projection.
Already at $\sigma/R\approx 1/6$, conditioning retains less overconfidence headroom than marginalization (variance ratio ${\approx}1.13$ vs.\ ${\approx}1.34$).
At $\sigma/R\approx 0.5$ the exact conditioned density exhibits a qualitatively different concentration pattern---different kurtosis and tail weight---compared to the tangent-chart approximation, illustrating density-shape artifacts that $W_2$ may underestimate (see Section~\ref{sec:discussion}).

\begin{figure}[t]
  \centering
  \includegraphics[width=\linewidth]{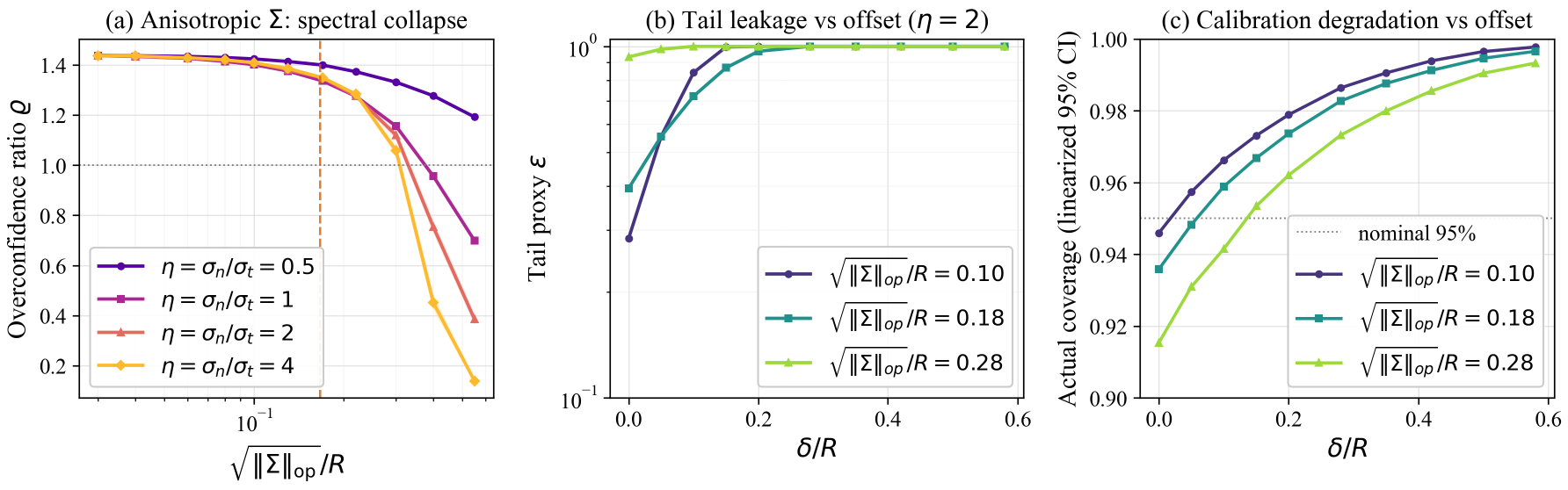}
  \caption{Circle generality stress tests:
  (a) anisotropic sweep ($\eta=\sigma_n/\sigma_t$) versus $\sqrt{\|\Sigma\|_{\mathrm{op}}}/R$,
  (b) tail proxy $\varepsilon$ versus offset $\delta/R$, and
  (c) coverage degradation versus $\delta/R$.}
  \label{fig:overconf-generality}
\end{figure}

\subsection{Planar pushing: diagnostics in a contact-rich task}\label{sec:numerics_planar_pushing}
\paragraph{Setup and covariance models}
We use the standard 2-D planar-pushing benchmark from constrained GTSAM/InCOpt to evaluate covariance extraction under hard contact constraints~\cite{qadri2022incopt,kaess2011isam2}. A rectangular box state
$x_k=(p_{x,k},p_{y,k},\theta_k)\in\mathbb R^2\times\mathbb S^1$
is pushed by a circular probe whose center follows a known deterministic path. The probe is assumed to remain in contact with the box boundary at every step, encoded as a per-pose contact factor $c_k(x_k)=0$. Odometry provides relative-pose measurements with wrapped-angle noise,
$u_k=x_{k+1}-x_k+\varepsilon_k$, $\varepsilon_k\sim\mathcal N(0,Q)$,
$Q=\mathrm{diag}(0.03^2,0.03^2,0.01^2)$.
We use $r_p=0.1$\,m, $(w,h)=(0.2,1.0)$\,m, and $n_{\mathrm{steps}}=50$.
The constrained mean is obtained from constrained least squares:
\[
\min_{x_{1:n}}\ \sum_{k=0}^{n-1}\|r_k(x)\|_{Q^{-1}}^2
\quad\text{s.t.}\quad c_k(x_k)=0,\ \forall k,
\]
with $r_k(x):=x_{k+1}-(x_k+u_k)$.
Linearizing about the optimized trajectory, we form the (unconstrained) Gauss--Newton information
$F\approx H^\top W^{-1}H$ with $W=\mathrm{blkdiag}(Q,\dots,Q)$, giving
$\Sigma_{\mathrm{unc}}=F^{-1}$, and compute the constrained covariance by conditioning onto the linearized contact constraints:
$\Sigma_{\mathrm{con}}=N(N^\top F N)^{-1}N^\top$, where $N\in\mathrm{null}(S^\top)$ spans the tangent directions of the constraint manifold.

\paragraph{Aggregate covariance reduction}
Across $49$ optimized timesteps, equality constraints are satisfied to numerical precision.
Table~\ref{tab:planar-pushing} reports representative timesteps:
the constrained $xy$-marginal covariance trace is reduced by $50.5\%$ at $k=1$ and $57.3\%$ at
$k=49$, with a peak reduction of $58.2\%$ at $k=37$.
These values establish a trajectory-level baseline before the Monte Carlo calibration diagnostics.

\begin{table}[t]
\centering
\caption{Quantitative covariance metrics for planar pushing at timesteps}
\label{tab:planar-pushing}
\footnotesize
\begin{tabular}{c|cc|cc|c|c}
\hline
 & \multicolumn{2}{c|}{$\mathrm{tr}(\Sigma_{xy})$} & \multicolumn{2}{c|}{$\lambda_{\max}(\Sigma_{xy})$} & Red. & $\|\Delta\Sigma_{xy}\|_F$ \\
$k$ & Unc. & Con. & Unc. & Con. & (\%) & \\
\hline
1 & 0.0018 & 0.0009 & 0.0009 & 0.0009 & 50.5 & 0.0009 \\
13 & 0.0234 & 0.0105 & 0.0117 & 0.0105 & 55.3 & 0.0118 \\
25 & 0.0450 & 0.0190 & 0.0225 & 0.0190 & 57.8 & 0.0228 \\
37 & 0.0666 & 0.0278 & 0.0333 & 0.0278 & 58.2 & 0.0337 \\
49 & 0.0882 & 0.0377 & 0.0441 & 0.0377 & 57.3 & 0.0446 \\
\hline
\end{tabular}
\end{table}

\paragraph{Trajectory-wide diagnostics}
To move beyond single-pose checks, we evaluate diagnostics at every timestep.
Let $c_k(x)=0$ be the contact constraint at step $k$, linearized at $\hat x_k$.
With $N_k$ spanning $\mathrm{null}\!\big(\nabla c_k(\hat x_k)^\top\big)$, we define
\begin{equation}\label{eq:curv-proxy}
\hat\kappa_k
:=
\frac{\left\|N_k^\top \nabla^2 c_k(\hat x_k) N_k\right\|_{\mathrm{op}}}
{\left\|\nabla c_k(\hat x_k)\right\|},
\qquad
\hat\rho_k
:=
\frac{1}{\max(\hat\kappa_k,\epsilon_\kappa)},
\end{equation}
as curvature/reach proxies ($\epsilon_\kappa=10^{-8}$ for numerical stability), and the uncertainty spread
\begin{equation}\label{eq:spread-def}
s_k:=\sqrt{\|\Sigma_{\mathrm{unc},k}\|_{\mathrm{op}}}.
\end{equation}
At each $k$, we compare tangent-plane covariance
$\Sigma_{\mathrm{lin},k}:=\Pi_k\Sigma_{\mathrm{unc},k}\Pi_k^\top$
against Monte Carlo covariance $\Sigma_{\mathrm{MC},k}$ after exact nonlinear projection, via
\begin{equation}\label{eq:mismatch-def}
\varrho_k^{\mathrm{MC}}
:=\frac{\tr(\Sigma_{\mathrm{MC},xy,k})}
{\tr(\Sigma_{\mathrm{lin},xy,k})},
\;
\Delta_k
:=\frac{\|\Sigma_{\mathrm{MC},xy,k}-\Sigma_{\mathrm{lin},xy,k}\|_F}
{\|\Sigma_{\mathrm{lin},xy,k}\|_F}.
\end{equation}
Figure~\ref{fig:planar_pushing_temporal} shows a non-uniform failure profile over time.
Larger mismatch co-occurs with larger locality-pressure indicators ($s_k/\hat\rho_k$),
consistent with the tail-leakage interpretation in Theorem~\ref{thm:W2-marg}.
The largest mismatch occurs at $k=13$
($\varrho_{13}^{\mathrm{MC}}=2.162$, locality index $s_{13}/\hat\rho_{13}=0.119$).
The median trajectory-wide mismatch is $\mathrm{median}(\varrho_k^{\mathrm{MC}})=1.557$ with
median relative Frobenius discrepancy $\mathrm{median}(\Delta_k)=0.525$.

\begin{figure}[t]
  \centering
  \includegraphics[width=0.9\linewidth]{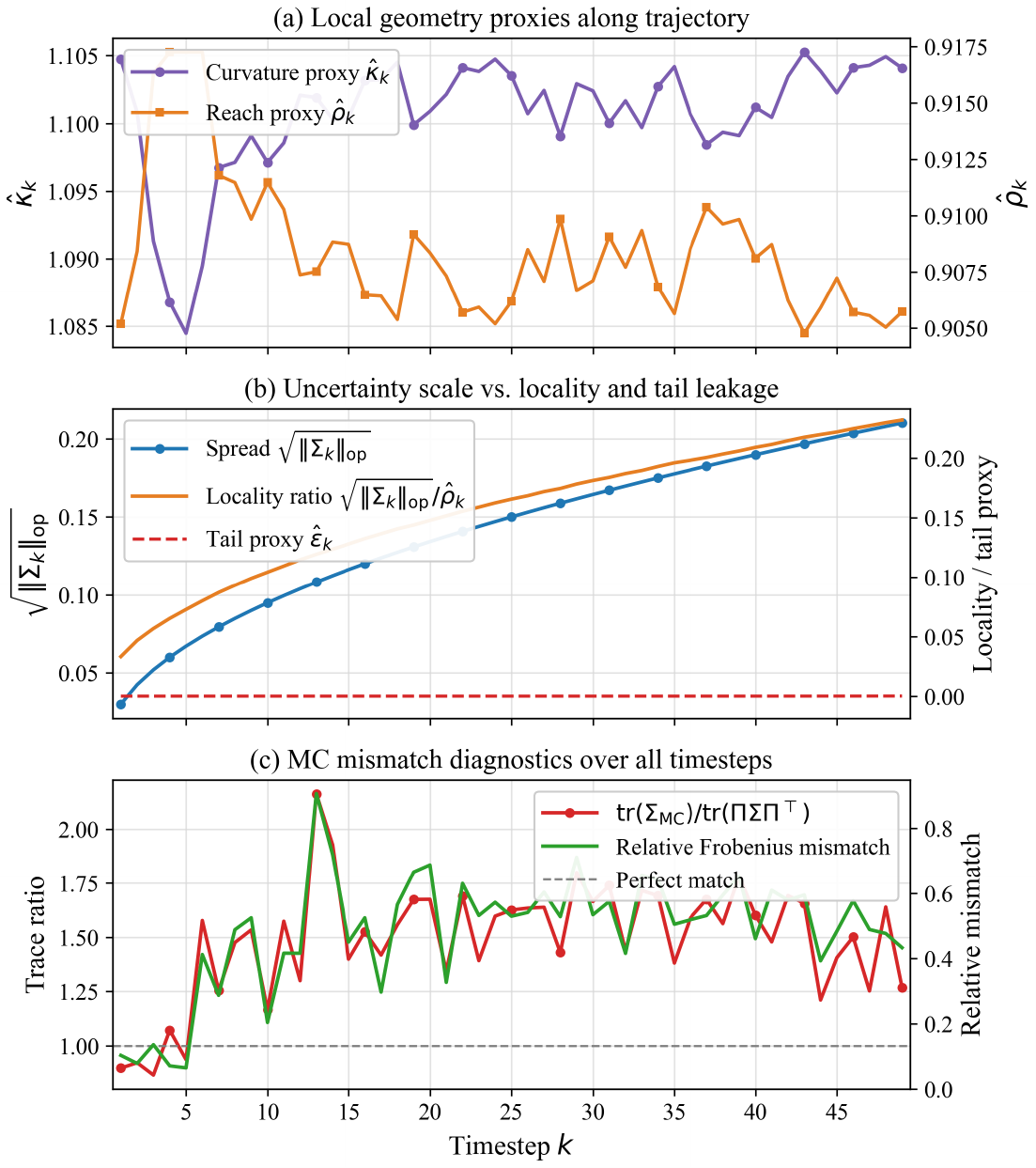}
  \caption{Trajectory-wide planar-pushing diagnostics:
  (a) curvature/reach proxies $(\hat\kappa_k,\hat\rho_k)$,
  (b) spread $s_k$ and locality indicator $s_k/\hat\rho_k$,
  (c) Monte Carlo mismatch $(\varrho_k^{\mathrm{MC}},\Delta_k)$.}
  \label{fig:planar_pushing_temporal}
\end{figure}

\paragraph{Directional uncertainty stress test}
Isotropic scaling $\Sigma\mapsto \alpha^2\Sigma$ obscures direction-specific failures.
At the terminal pose $k^\star$, with local unit normal $\hat n$, define
$A(\alpha_n,\alpha_t):=\alpha_t(I-\hat n\hat n^\top)+\alpha_n \hat n\hat n^\top$
and $\Sigma(\alpha_n,\alpha_t):=A\,\Sigma_{\mathrm{unc},k^\star}\,A^\top$.
We compare isotropic $(\alpha,\alpha)$, normal-only $(\alpha,1)$, and tangential-only $(1,\alpha)$ sweeps.
Figure~\ref{fig:planar_pushing_directional} shows a clear asymmetry for $\alpha\ge 1$:
normal-direction inflation causes the largest calibration and covariance mismatch.
This normal-tangent asymmetry provides experimental confirmation of the mechanism in Theorem~\ref{thm:W2-marg}: because $\|\Sigma\|_{\mathrm{op}}$ captures the worst-case (normal) spread, the bound tightens when the dominant eigenvalue aligns with the constraint normal, while remaining safely conservative for tangent-dominant uncertainty.

\begin{figure}[t]
  \centering
  \includegraphics[width=\linewidth]{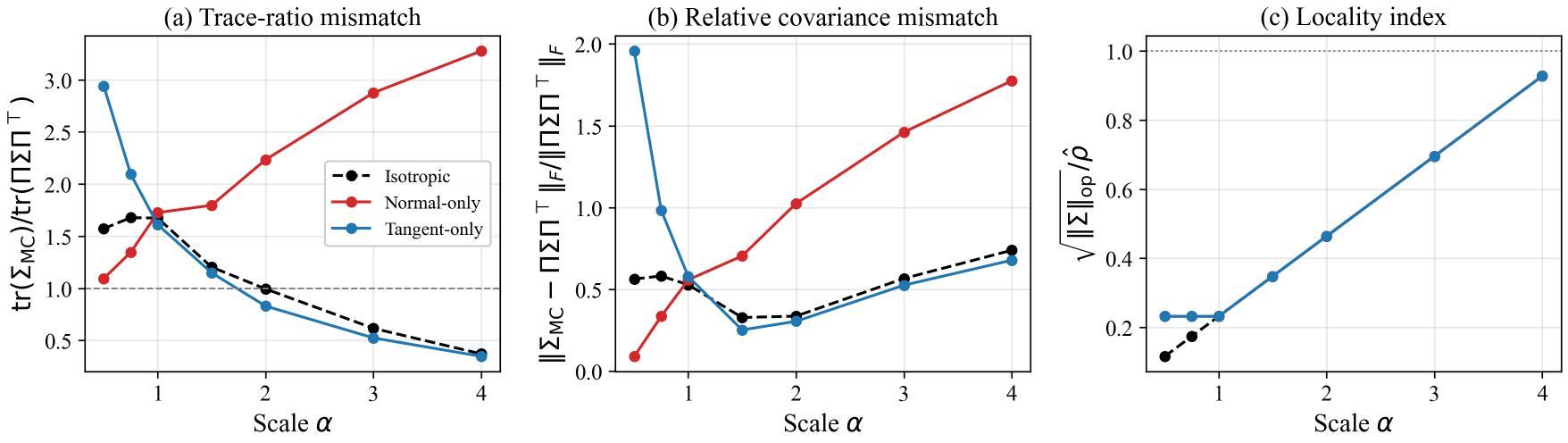}
  \caption{Directional scaling at the terminal planar-pushing pose:
  isotropic, normal-only, and tangential-only covariance inflation.}
  \label{fig:planar_pushing_directional}
\end{figure}

\section{Discussion}\label{sec:discussion}
\paragraph{Practical validity test for robotics estimation}
The bounds give a runtime calibration gate for factor-graph pipelines.
At each linearization point $\tilde\mu$, compute $k_\rho:=\sqrt{\|\Sigma\|_{\mathrm{op}}}/\rho$ and $s_\delta:=\|\delta\|/\rho$ from the marginal covariance, local reach/curvature, and linearization offset.
Curvature is analytic for $S^3$, $\mathrm{SO}(3)$, and $\mathrm{SE}(3)$, or can be approximated from the constraint Jacobian/Hessian as in Section~\ref{sec:numerics_planar_pushing}.
If $k_\rho\lesssim 1/6$ and $s_\delta\ll 1$, the tangent approximation is reliable; otherwise the decomposition suggests iterated relinearization for large $k_\rho$ with small $s_\delta$~\cite{bell1993iterated}, immediate relinearization for large $s_\delta$ (e.g.\ after a loop closure), and multi-chart~\cite{cesic2017mixture,nascimento2014manifold} or particle~\cite{gordon1993novel,koval2017manifold} methods when both are large.
For conditioning, $\eta_r\ll 1$ is an additional guard against chart-distortion artifacts that $W_2$ may underestimate.
The cost is one dominant-eigenvalue query plus a curvature lookup per variable; in GTSAM/iSAM2-style solvers, the covariance is available from the Bayes tree and curvature from analytic formulas or the constraint Hessian.
Unlike chi-squared innovation tests, which detect residual inconsistency after measurement incorporation~\cite{barfoot2024state}, these diagnostics predict linearization failure before the update and enable preemptive method switching.

\paragraph{$W_2$ scope and conditioning asymmetry}
Because $W_2$ measures mass displacement rather than density shape, similar bounds for marginalization and conditioning do not imply equal robustness.
Conditioning is more delicate: normalization constants such as $Z_T^{-1}$ in~\eqref{eq:QTDef} can amplify shape errors that $W_2$ does not penalize.
For example, conditioning a Gaussian onto $S^3$ at moderate $k_\rho$ may produce a narrow spherical cap with angular variance different from the tangent-Gaussian prediction, while $W_2$ remains moderate because the supports and centroids are close.
Projection-based marginalization is less sensitive to this artifact because the projection map fixes the support before normalization, making $W_2$ a more faithful calibration proxy there.
Conditioning should therefore be supplemented with coverage or KL checks when $k_\rho\gtrsim 0.1$.

\paragraph{Concrete regime examples}
For unit-quaternion orientation on $S^3$ ($\kappa=1$, $\rho=1$), $k_\rho=\sqrt{\|\Sigma\|_{\mathrm{op}}}$: IMU-grade uncertainty ($\lesssim 0.03$) is well below $1/6$, whereas prolonged dead reckoning ($\gtrsim 0.17$) exceeds it.
For planar contact with $R_c\approx 0.1$\,m, the threshold $\sqrt{\|\Sigma\|_{\mathrm{op}}}\lesssim R_c/6\approx 17$\,mm is compatible with calibrated sensors but can be violated by process-noise accumulation, as in Section~\ref{sec:numerics_planar_pushing}.
Loop-closure drift gives the complementary offset trigger: ${\sim}\,1\%$ drift implies $\|\delta\|\approx 0.5$\,m after $50$\,m, far beyond the locality radius.
For $\mathrm{SE}(3)$ with small rotation--translation cross-covariance~\cite{barfoot2024state}, rotational uncertainty follows the chosen $\mathrm{SO}(3)$ or unit-quaternion embedding constants, translational constraints inherit $k_\rho=\sqrt{\|\Sigma_{\mathrm{pos}}\|_{\mathrm{op}}}/R_c$, and the leading-order diagnostic is the maximum of both.

\paragraph{Interaction between offset and anisotropy}
In incremental solvers, $\|\delta\|$ and $k_\rho$ often grow together because state lag and covariance growth accumulate simultaneously.
When the dominant eigenvector of $\Sigma$ aligns with the offset, the cross-term $4\delta^\top\Sigma\delta$ in~\eqref{eq:GaussianFourthMoment} amplifies the local bound; if the offset is orthogonal to the dominant spread, only the additive $\|\delta\|^4$ term remains.
The anisotropic circle experiments confirm that coverage degrades fastest when $k_\rho$ and $s_\delta$ are large in the same direction.

\paragraph{Relation to nonlinear filtering methods}
The same decomposition clarifies what different filters improve.
Iterated Kalman filtering~\cite{bell1993iterated} reduces $s_\delta$ by re-evaluating at posterior modes, while invariant~\cite{bonnable2009invariant} and equivariant~\cite{van2022equivariant} filters exploit symmetry to reduce effective curvature on groups such as $\mathrm{SO}(3)$.
Thus $k_\rho$ and $s_\delta$ act as post-hoc certificates for the tangent approximation.
If they remain above threshold after filtering, multi-chart or particle methods are warranted regardless of filter type.

\paragraph{Limitations}
The constants are conservative (general coupling/concentration inequalities rather than manifold-specific transport).
Specialized transport constructions, such as geodesic couplings on Lie groups, may tighten the constants and reduce the theory--experiment gap in Section~\ref{sec:numerics}.
The bounds are also per-variable and single-step, whereas coupled robotics constraints may compound tangent-approximation errors across variables and long-horizon cascades may require joint diagnostics.
The analysis assumes $C^2$ manifolds with positive reach; piecewise-smooth contact manifolds, joint multi-variable diagnostics, and non-Gaussian priors remain future work.

\section{Conclusion}\label{sec:conclusion}
We derived explicit non-asymptotic $W_2$ stability bounds for tangent-linearized Gaussian marginalization and conditioning on smooth manifolds, decomposing error into local geometric distortion and tail leakage.
For Gaussian inputs, the bounds yield closed-form diagnostics from $(\mu,\Sigma)$ and local curvature/reach, confirmed by circle and planar-pushing benchmarks that locate the calibration transition at $k_\rho\approx 1/6$ with normal-direction uncertainty as the dominant failure mode.
A structural asymmetry emerges: conditioning degrades faster than marginalization because $Z_T^{-1}$ amplifies density-shape distortions invisible to $W_2$, while the mean-offset diagnostic $s_\delta$ provides a complementary escalation trigger for factor-graph solvers to adaptively switch between relinearization, multi-chart, and particle methods.
Future extensions include piecewise-smooth geometries, joint multi-variable diagnostics, and non-Gaussian priors.

\section*{Acknowledgment}
This work was supported by Deep-Tech TIPS funded by the Ministry of SMEs and Startups.

\balance
\bibliographystyle{IEEEtranN}
\bibliography{references.bib}

\end{document}